\documentclass[table,11pt]{article}
\usepackage[top=1in, bottom=1in, left=1in, right=1in]{geometry}

\usepackage[linktocpage,colorlinks,linkcolor=blue,anchorcolor=blue,citecolor=blue,urlcolor=blue,pagebackref]{hyperref}
\usepackage{euscript,mdframed}
\usepackage{microtype,todonotes,relsize}
\usepackage{amsmath,amsthm}

\usepackage{accents}
\usepackage{comment,url,graphicx,relsize}
\usepackage{amssymb,amsfonts,amsmath,amsthm,amscd,dsfont,mathrsfs,mathtools,nicefrac, bm}
\usepackage{float,psfrag,epsfig,color,url}
\usepackage{epstopdf,bbm,mathtools,enumitem}
\usepackage{subfigure}
\usepackage[ruled,vlined]{algorithm2e}
\usepackage{tablefootnote}
\graphicspath{{Plots/}}
\usepackage{diagbox}
\usepackage{tikz}
\usetikzlibrary{calc,patterns,angles,quotes}
\usepackage{makecell}
\usepackage{multirow}
\usepackage{booktabs}

\usepackage{xcolor}

\newcommand{\E}{\mathbb{E}}
\newcommand{\RR}{\mathbb{R}}

\ifdefined\nonewproofenvironments\else
\ifdefined\ispres\else
\renewenvironment{proof}{\noindent\textbf{Proof.}\hspace*{.3em}}{\qed\\}
\newenvironment{proof-sketch}{\noindent\textbf{Proof Sketch}
  \hspace*{0.em}}{\qed\bigskip\\}
\newenvironment{proof-idea}{\noindent\textbf{Proof Idea}
  \hspace*{0.em}}{\qed\bigskip\\}
\newenvironment{proof-of-lemma}[1][{}]{\noindent\textbf{Proof of Lemma {#1}.}
  \hspace*{0.em}}{\qed\\}
\newenvironment{proof-of-corollary}[1][{}]{\noindent\textbf{Proof of Corollary {#1}.}
  \hspace*{0.em}}{\qed\\}
\newenvironment{proof-of-theorem}[1][{}]{\noindent\textbf{Proof of Theorem {#1}.}
  \hspace*{0.em}}{\qed\\}
\newenvironment{proof-attempt}{\noindent\textbf{Proof Attempt}
  \hspace*{0.em}}{\qed\bigskip\\}

\fi
\newtheorem{theorem}{Theorem}[section]
\newtheorem{lemma}{Lemma}[section]

\usetikzlibrary{calc}

\allowdisplaybreaks
\usepackage[authoryear,round]{natbib}
\renewcommand*{\backref}[1]{\ifx#1\relax \else Page #1 \fi}
\renewcommand*{\backrefalt}[4]{%
  \ifcase #1 \footnotesize{(Not cited.)}%
  \or        \footnotesize{(Cited on page~#2.)}%
  \else      \footnotesize{(Cited on pages~#2.)}%
  \fi
}

\newcommand*{\colorboxed}{}
\def\colorboxed#1#{%
  \colorboxedAux{#1}%
}
\newcommand*{\colorboxedAux}[3]{%
  \begingroup
    \colorlet{cb@saved}{.}%
    \color#1{#2}%
    \boxed{%
      \color{cb@saved}%
      #3%
    }%
  \endgroup
}
\numberwithin{equation}{section}
\usepackage{marginnote}

\newcommand{\todol}[2][]{{%
 \let\marginpar\marginnote
 \reversemarginpar
 \renewcommand{\baselinestretch}{0.8}%
 \todo[color=yellow]{#2}}}

\usepackage[utf8]{inputenc}
\usepackage[normalem]{ulem}
\usepackage{algorithmic}
\usepackage{wrapfig}

\title{Scalable Parameter and Memory Efficient Pretraining for LLM: Recent Algorithmic Advances and Benchmarking}

\author{
{Athanasios Glentis} \thanks{Department of Electrical and Computer Engineering, University of Minnesota. Equal contribution. \texttt{ glent007@umn.edu}}
\and
{Jiaxiang Li} \thanks{Department of Electrical and Computer Engineering, University of Minnesota. Equal contribution.\texttt{ li003755@umn.edu}}
\and
{Qiulin Shang} \thanks{Peking University. Equal contribution. \texttt{qiulin13145@gmail.com}}
\and
{Andi Han} \thanks{School of Mathematics and Statistics, University of Sydney. \texttt{andi.han@sydney.edu.au }}
\and
{Ioannis Tsaknakis} \thanks{Department of Electrical and Computer Engineering, University of Minnesota. \texttt{tsakn001@umn.edu}}
\and
{Quan Wei} \thanks{Department of Electrical and Computer Engineering, University of Minnesota. \texttt{wei00355@umn.edu}}
\and
{Mingyi Hong} \thanks{Department of Electrical and Computer Engineering, University of Minnesota. \texttt{mhong@umn.edu}}
}
\date{}

\begin{document}
\maketitle

\begin{abstract}
    Fueled by their remarkable ability to tackle diverse tasks across multiple domains, large language models (LLMs) have grown at an unprecedented rate, with some recent models containing trillions of parameters. This growth is accompanied by substantial computational challenges, particularly regarding the memory and compute resources required for training and fine-tuning. Numerous approaches have been explored to address these issues, such as LoRA. While these methods are effective for fine-tuning, their application to pre-training is significantly more challenging due to the need to learn vast datasets. Motivated by this issue, we aim to address the following questions: {\bf Can parameter- or memory-efficient methods enhance pre-training efficiency while achieving performance comparable to full-model training? How can the performance gap be narrowed?} To this end, the contributions of this work are the following. (1) We begin by conducting a comprehensive survey that summarizes state-of-the-art methods for efficient pre-training. (2) We perform a benchmark evaluation of several representative memory efficient pre-training approaches to comprehensively evaluate their performance across model sizes. We observe that with a proper choice of optimizer and hyperparameters, full-rank training delivers the best performance, as expected. We also notice that incorporating high-rank updates in low-rank approaches is the key to improving their performance. (3) Finally, we propose two practical techniques, namely weight refactorization and momentum reset, to enhance the performance of efficient pre-training methods. We observe that applying these techniques to the low-rank method (on a 1B model) can achieve a lower perplexity than popular memory efficient algorithms such as GaLore and Fira, while simultaneously using about $25\%$ less memory.
    The code used for the experiments is provided in \url{https://github.com/OptimAI-Lab/Memory_Efficient_Pretraining}.
\end{abstract}

\section{Introduction}

The size of large language models (LLMs) has grown exponentially in recent years, driven by their unprecedented ability to generate high-quality text, solve complex problems, and perform a wide range of tasks across domains. For instance, GPT-3 contains 175 billion parameters~\citep{ye2023comprehensive}, and more recent models like GPT-4 and PaLM-2 have further pushed the boundary, with hundreds of billions to trillions of parameters~\citep{achiam2023gpt,anil2023palm}. This rapid scaling has, however, brought significant computational challenges, particularly in terms of the memory and compute resources required for both training and deployment.

While pre-training these large models remains computationally intensive, many efforts have focused on making downstream fine-tuning more efficient. Parameter-Efficient Fine-Tuning (PEFT) techniques such as Low-Rank Adaptation (LoRA, \cite{hu2022lora}) have been particularly impactful. LoRA introduces a small number of trainable low-rank matrices to adapt pre-trained weights without modifying the original model parameters. This approach significantly reduces the number of trainable parameters and memory footprint during fine-tuning, enabling practitioners to adapt large models for specific tasks on resource-constrained scenarios. Other methods, such as compression~\citep{zhu2024survey}, pruning~\citep{sun2023simple,ma2023llm} and quantization~\citep{egashira2024exploiting,liu2024spinquant}, further optimize this process.

While parameter-efficient techniques like LoRA excel in fine-tuning scenarios, extending these methods to pre-training poses a fundamental challenge: pre-training involves processing massive datasets over billions of tokens, which requires the model to learn comprehensive representations, whereas PEFT techniques such as LoRA significantly limit parameter updates to smaller spaces. As a result, it is hard to directly adapt these methods to improve pre-training efficiency while still maintaining a comparable performance with the full-model training.

Recently, a growing number of studies on parameter and memory efficient pre-training have emerged. {Notably, memory-efficient optimizers, such as GaLore~\citep{zhao2024galore,su2025galore} (which reduces optimizer state memory using low-rank projections) still fall short of full-rank pre-training due to the substantial loss of gradient information.} 
GaLore draws substantial attention from the community due to its effectiveness in maintaining a full-rank model while significantly reducing the memory required for storing optimizer states (see Section \ref{sec:review} for details). All the studies in this direction aim at a fundamental research question:

\begin{center}
\begin{minipage}{40em}
  \textit{Can parameter- or memory-efficient methods achieve performance comparable to full-rank pretraining? If not, how can the performance gap be narrowed?}
\end{minipage}
\end{center}

{The main objective of this work is to address this question. To achieve this goal we make the following contributions: 
\begin{enumerate}
    \item First, we conduct a survey to provide a systematic summary and comparison of recent advances in efficient pre-training strategies. To maintain focus and clarity, we only include works that either target pre-training specifically or have verified their applicability in enhancing parameter- or memory-efficient pre-training.
    
    \item Next, we present a comprehensive benchmark\footnote{The code for the benchmark is publicly available at: \url{https://github.com/OptimAI-Lab/Memory_Efficient_Pretraining}.} evaluation of several representative LLM pre-training methods across different models sizes. The results particularly highlight the \textit{importance of full-rankness in pre-training}. 
    \begin{enumerate}
        \item \underline{First}, we show that the full-rank pre-training still achieves the best performance with proper optimizer and hyperparameters. 

        \item \underline{Secondly}, We show that the (vanilla) low-rank method (where the weight matrix is factorized as $BA$ and both $A$ and $B$ are updated) can be used for pre-training with certain performance drop, contradicting recent studies such as \cite{zhao2024galore} where low-rank method fails completely. 

        \item \underline{Thirdly}, we observe that restoring full-rankness in low-rank methods significantly boosts the performance of parameter-efficient pre-training. For instance, Fira, a low-rank method which involves high-rank updates, yields better performance than lower-rank update methods such as GaLore.
    \end{enumerate}
    
    \item Finally, we propose two practical techniques, refactorization and momentum reset, that substantially enhance the performance of parameter-efficient pre-training, towards matching the performance of full-rank models. We observe that applying these techniques to the low-rank method (on a 1B model) allows us to achieve lower perplexity than GaLore and Fira, while simultaneously using about $25\%$ less memory.
\end{enumerate}
}

\section{Review on parameter-efficient pre-training methods}\label{sec:review}

\begin{figure}[t]
    \begin{center}
    \includegraphics[width=0.6\textwidth]{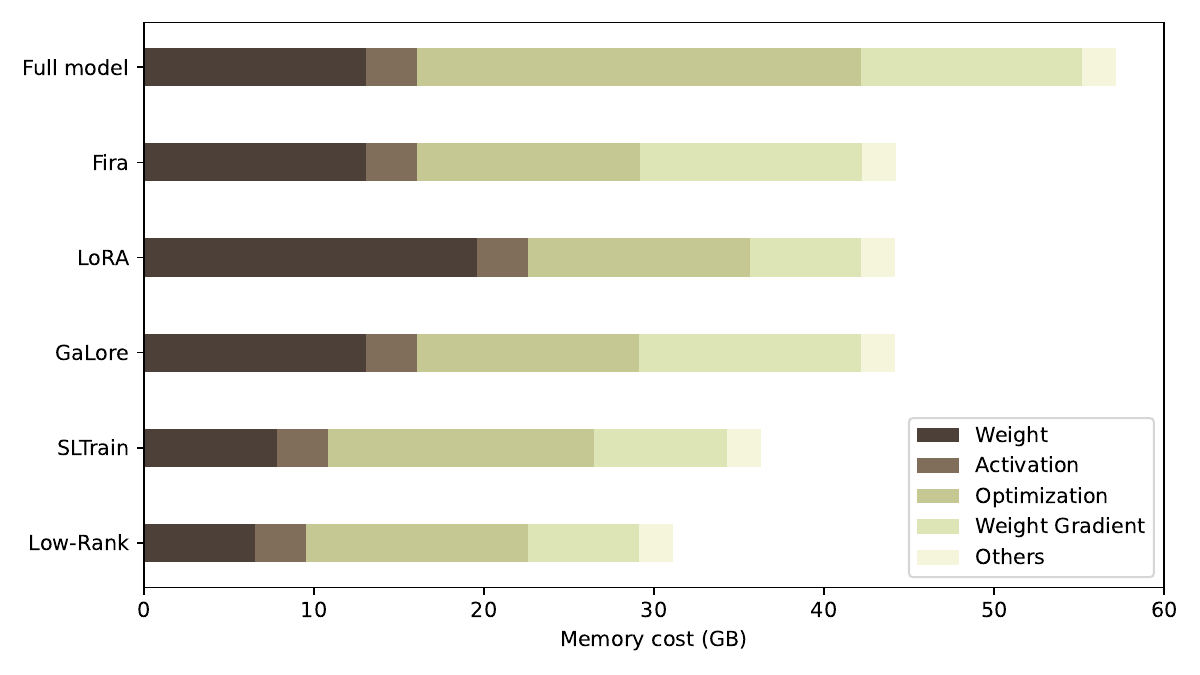}
    \caption{\footnotesize Estimated memory consumption of pre-training and inference for a 7B model with a token batch size of 256 on a single device. All methods use BF16 format and AdamW (without optimizer and activation checkpointing). For low-rank, LoRA and GaLore we assume a rank of $512$ (with the original attention head dimension as $2048$), and for SLTrain we use rank $512$ and $\delta=0.1$. The weight factorization methods (low-rank and SLTrain) save more memory compared to memory-efficient optimizers (GaLore and Fira), with a possible compromise on the performance (See Section \ref{sec:bench_res} for the results). 
    }
    \label{fig:mem_category}
    \end{center}
\end{figure}

In this section we review existing works on parameter/memory-efficient pre-training. From Figure \ref{fig:mem_category}, the memory required for training can be roughly categorized into the following categories: 1. memory for storing \textit{model weights}; 2. memory for \textit{activations}; 3. memory for parameter \textit{gradients}; 4. memory for \textit{optimizer states}, such as momentum; 5. other memories, such as GPU system memory. Various methods have been proposed, focusing on reducing different categories of the memory. \underline{First}, a large number of works focus on proposing new \textbf{optimizers} for parameter/memory-efficient training, usually leading to the reduction of the memory of gradient and optimizer states; \underline{Second}, another line of works focus on \textbf{factorizing} the weight matrices, reducing the effective parameters that need to be trained, leading to reduced memory utilization as well; \underline{Third}, {other methods mostly based on system-level or low-level innovations, such as compression, quantization, etc}. In the following sub-sections, we will discuss these three categories of methods in detail.

\begin{table}
  \centering
\resizebox{0.90\textwidth}{!}{%
\begin{tabular}{c|cccc|c}
    \toprule
    Method & Weight & Activation & Optimizer & Gradient & Comments \\
    \midrule
    Galore \citep{zhao2024galore} &  &  & \checkmark &  \\
    ReLoRA \citep{lialin2024relora} &  &  &  \checkmark &  \checkmark \\
    SwitchLoRA \citep{zhou2024revolutionizing} &  &  & \checkmark & \checkmark \\
    SLTrain \citep{han2024sltrain} & \checkmark &  & \checkmark & \checkmark \\
    \citep{wei2024building} & \checkmark &  & \checkmark & \checkmark \\
    {BTT} \citep{qiu2024compute} & \checkmark &  & \checkmark & \checkmark \\
    {CoMERA} \citep{yang2024comera} & \checkmark &  & \checkmark & \checkmark \\
    {LORO} \citep{mo2025parameter} & \checkmark &  & \checkmark & \checkmark \\
    NeuZip \citep{hao2024neuzip} & \checkmark & &  &  \\
    LPA \citep{lv2024scalable} & \checkmark &  & \checkmark & \checkmark \\
    {Fira} \citep{chen2024fira} & &  & \checkmark & \checkmark \\
    {Natural Galore} \citep{das2024natural} & &  & \checkmark &  \\
    {Q-Galore} \citep{zhang2024q} & &  & \checkmark &  & Quantize \\
    COAP \citep{xiao2024coap} & &  & \checkmark &  \\
    {CompAct} \citep{shamshoum2024compact} & & \checkmark & \checkmark & \checkmark \\
    {GRASS} \citep{muhamed2024grass} & &  & \checkmark & \checkmark \\
    {LISA} \citep{pan2024lisa} & &  & \checkmark & \checkmark \\
    {BlockLLM} \citep{ramesh2024blockllm}  & &  & \checkmark & \checkmark \\
    {LoQT} \citep{loeschcke2024loqt} & &  & \checkmark & \checkmark \\
    {BAdam} \citep{luo2024badam} & & & \checkmark & \checkmark \\
    {Adam-mini} \citep{zhang2024adam} & &  & \checkmark & \\
    {LDAdam} \citep{robert2024ldadam} & &  & \checkmark & \\
    {APOLLO} \citep{zhu2024apollo} & &  & \checkmark & \checkmark \\
    {FRUGAL} \citep{zmushko2024frugal} & &  & \checkmark & \checkmark \\
    {Tensor-Galore} \citep{george2025tensorgalore} & &  & \checkmark &  \\
    {H-Fac} \citep{nguyen2024h} & &  & \checkmark & \\
    {Galore+} \citep{liao2024galore} &  &  & \checkmark & \checkmark \\
    {SubTrack-Grad} \citep{rajabi2025subtrack} &  &  & \checkmark & \\
    {Online subspace descent} \citep{liang2024memory} &  &  & \checkmark & \\
    {Flora} \citep{hao2024flora} &  &  & \checkmark & \\
    {GoLore} \citep{he2024subspace} &  &  & \checkmark & \\
    {I3S} \citep{zhang2025i3s} &  &  & \checkmark & \\
    {4-bit Shampoo} \citep{wang20244} &  &  & \checkmark & \checkmark &  \\
    {AdaLoRA} \citep{zhang2023adalora} &  &  & \checkmark & \checkmark &  \\
    {CoLA} \citep{liu2025cola} & \checkmark & \checkmark & \checkmark & \checkmark &  \\
    {VeLoRA} \citep{miles2024velora} &  &  & \checkmark &\checkmark &  \\
    {CoLM} \citep{nguyen2024mini} &  &  &  & & Data reduction \\
    \bottomrule
  \end{tabular}
}
  \caption{\footnotesize Summary of related works. ``\checkmark'' indicates that the method reduces the memory in the corresponding category during pre-training. Note that gradient and activation checkpointing are not considered when making this table, since all the mentioned methods could save gradient and activation memory seamlessly by checkpointing. Besides, works on quantizing existing methods are marked ``Quantize'' in the comments}
  \label{table:related_works}
\end{table}

\subsection{Memory-efficient optimizers}
Innovation in optimizers is perhaps the most effective way to reduce the memory consumption of optimizer states and gradients, and it is perhaps the most popular on-going research direction. In this section, we review a number of such innovations. Note that we are not including commonly-used infra tricks such as gradient checkpointing. In fact, all the following methods can further save their gradient and optimizer state memory by checkpointing in practice.

\textbf{GaLore} \citep{zhao2024galore} is perhaps the most representative approach. It projects the gradients into a low-dimensional space to reduce memory consumption for both gradients and optimizer states. The subspace is determined by the singular value decomposition of the gradient momentum to maximally capture the full gradient information. 
{
To describe the algorithm, suppose we consider a simple problem of optimizing the loss function $\ell(W)$ with one weight matrix $W\in\RR^{m\times n}$ (without loss of generality assume $m\geq n$). GaLore proposes projecting the optimizer states of Adam and performing updates through the following iterations: 
(we present only the key steps, omitting some of the typical Adam steps)
\begin{equation}\label{eq:algo_random_proj_adam_matrix}
\begin{aligned}
& G_t \leftarrow \nabla \ell\left(W_{t}\right)\\
& R_t \leftarrow (P_t)^\top G_t, \\
& M_t \leftarrow \beta_{1 t} M_{t-1}+\left(1-\beta_{1 t}\right) R_t, \\
& V_t \leftarrow \beta_{2 t} V_{t-1}+\left(1-\beta_{2 t}\right) R_t^2, \\
& W_{t+1} \leftarrow W_{t}-\alpha_t P_t \left(\frac{M_{t}}{\sqrt{{V}_t}}\right),
\end{aligned}
\end{equation}
where $P_t\in\RR^{m\times r}$ is an orthogonal projection to $r$-dimensional space with $r\leq n$, obtained by conducting singular value decomposition (SVD) to the gradient matrix $[U,\Sigma, V] = \text{SVD}(G_t)$ and taking $P_t=U_{:,1:r}$. Note that the momenta $M_t$ and $V_t$ are now in low-dimensional space $\RR^{r\times n}$ so that the memory of the optimizer states can be saved substantially. 
}

\textbf{Fira} \citep{chen2024fira} achieves a full-rank gradient update by adding an extra orthogonal term to the projected gradient of GaLore, effectively re-introducing the full rank information and achieving better performance than GaLore.
{
Specifically,  the update rule of Fira is given by replacing the last line in \eqref{eq:algo_random_proj_adam_matrix} with the following update
\begin{equation}
    W_t \leftarrow W_{t-1}-\alpha_t P_t \left(\frac{M_{t}}{\sqrt{{V}_t}}\right) - \underbrace{\alpha_t \phi_t(R_t) (G_t - P_t R_t)}_{\text{Extra term}},
\end{equation}
where $\phi_t$ is a rescaling scalar function to balance the weight of the newly introduced term and the GaLore gradient term. The introduction of this extra term plays an important role since it re-introduces the gradient information outside of the subspace defined by $P_t$ (note that no extra memory is required). 
}

Next, we introduce a few other works that extends the above projection based ideas. \textbf{Natural Galore} \citep{das2024natural} combines preconditioning via natural gradient with the GaLore scheme \eqref{eq:algo_random_proj_adam_matrix}; \textbf{COAP} \citep{xiao2024coap} achieves better performance than Galore via a correlation-aware design on the projection matrix; \textbf{Gradient structured sparsification (GRASS)} \citep{muhamed2024grass} designs sparse projection matrices guided by the norm of the rows of the gradients, and achieves even faster pre-training with competitive performance to GaLore; \textbf{FRUGAL} \citep{zmushko2024frugal} further refines the subspace gradient update, as it does a state-full (such as Adam) step on the projection and an extra state-free (such as SGD) step on the orthogonal components.  This enables updates across the full parameter space and achieves better performance without incurring significantly more memory overhead than GaLore; \textbf{H-Fac} \citep{nguyen2024h} uses a rank-1 parameterization on optimizer states to update the weights in a memory efficient manner, inspired by a continuous Hamiltonian update scheme; \textbf{Tensor-Galore} \citep{george2025tensorgalore} extends GaLore to tensor variables to reduce the memory for all weight parameters than matrix variables only; \textbf{Galore+} \citep{liao2024galore} partially reduces the computational cost of SVD by performing it on only one randomly selected attention matrix (among all the heads in multi-head attention), leading to significant training acceleration at the expense of introducing slightly more error; \textbf{SubTrack-Grad} \citep{rajabi2025subtrack} and \textbf{Online subspace descent} \citep{liang2024memory} are two methods that update the subspace of GaLore via solving optimization problems (which approximate the current gradient in a least square sense); \textbf{Flora} \citep{hao2024flora} and \textbf{GoLore} \citep{he2024subspace} propose randomly projecting the gradients to a low-dimensional space to reduce memory usage, while still maintaining competitive performance compared to GaLore; \textbf{I3S} \citep{zhang2025i3s} dynamically adjusts GaLore’s subspace via importance sampling (weighted by the singular values) to achieve improved performance over GaLore. 

Beside all the aforementioned projection based method, substantial progresses have been made in block-wise type of optimizers (mostly based on Adam). \textbf{BAdam} \citep{luo2024badam} is a block coordinate descent framework that uses Adam. Specifically, {it divides the model parameters $W$ into $D$ blocks $W^{b_{1}}, \dots, W^{b_{D}}$ and updates one block at a time using Adam's update steps,  i.e., for the $i$th block we approximately solve the following problem using $K$ Adam steps
\begin{align*}
    W_{t+1}^{b_{i}} \leftarrow  \arg\min\limits_{W^{b_{i}}} \ell \left(W_{t+1}^{b_{1}}, \dots,W_{t+1}^{b_{i-1}},W^{b_{i}},W_{t}^{b_{i+1}}, \dots,W_{t}^{b_{D}}\right).
\end{align*}%
}%
At the end of each block update, it resets the gradient and optimization states, reducing that way the memory requirements. \textbf{Adam-mini} \citep{zhang2024adam} achieves memory savings by reducing the number of learning rates, as a different rate is typically required for each parameter. First, it partitions the model parameters into blocks by leveraging the block-diagonal structure of the problem's Hessian. Then, it assigns a single learning rate to each parameter block by substituting the quantity {$G^{b}\odot G^{b}$ (where $G^{b}$ is the gradient of the loss over block $b$) by its average value within that block, i.e.,  
\begin{align*}
    W^{b} \leftarrow (1-\beta_{2})\cdot\text{mean}(G^{b}\odot G^{b}) + \beta_{2}\cdot W^{b}
\end{align*}
where $W^{b}$ is the second raw moment estimate for the block $b$ and $\beta_{2}$ is the exponential decay rate.
}%
\textbf{BlockLLM} \citep{ramesh2024blockllm} also achieves memory efficiency by selecting a small subset of training parameter for Adam update, where the selection criterion is measuring the norm of the gradients for each weight matrix; \textbf{LISA} \citep{pan2024lisa} is a similar block-wise update and is inspired by the observation of a skewed distribution of weight norms in LoRA updates across various layers. To simulate LoRA, LISA approach randomly (with importance sampling) freezes most of the middle layers during optimization, while updating the remaining layers using the Adam optimizer; \textbf{LDAdam} \citep{robert2024ldadam} is an adaptive optimizer that reduces memory requirements by using low-rank compression for optimizer states. The key idea of the algorithm is to optimize within lower-dimensional subspaces but still explore the entire space of parameters during training.
Leveraging the redundancy in AdamW’s learning rate adaptation rule, \textbf{APOLLO} \citep{zhu2024apollo} restructures its element-wise learning rate updates into a channel- or tensor-wise format, where each channel/tensor shares the same scaling factor. APOLLO employs a memory-efficient approximation of these scaling factors via an auxiliary optimizer state that relies only on low-dimensional gradient information, resulting in significant memory savings.

\subsection{Weight Factorizations}
Now we turn to the second category of the discussion, which are the methods that reduce memory of weights and activations. We again do not consider activation checkpointing as it could be applied to all of the following methods seamlessly.

We start our review from low-rank idea derived from LoRA \citep{hu2022lora}. In LoRA, each of the weight matrix $W\in\RR^{m\times n}$ (which could be the attention matrix $Q$, $K$ and $V$) is re-parameterized as
\begin{equation*}
    W = W_0+BA
\end{equation*}
where $B\in \RR^{m\times d}$ and $A\in\RR^{d\times n}$ with $r\ll\min\{m,n\}$. For fine-tuning, $W_0$ is the weight of the pre-trained model; For pre-training, the pre-trained model $W_0$ is no longer available. In this case, one could replace $W_0$ with the zero matrix, in which case the method  will be called the \textbf{low-rank} approach  (see for example \cite{kamalakara2022exploring,lv2024scalable}); On the other hand, one can also replace $W_0$ with a randomly initialized matrix, in which case the method will still be referred to as the \textbf{LoRA} approach. Note that the low-rank approach where $W=BA$, parameterizes the given weight matrix completely within a low-rank space (with only $(m+n)r$ parameters), is the most natural idea for reducing the number of training parameters. However it is observed that this plain low-rank approach cannot match the performance of the full rank model~\citep{zhao2024galore}. A plausible explanation is that low-rank models do not have the expressiveness of full rank models. 

Improving LoRA to achieve comparable performance as its full-rank counterpart is thus an active research direction. \textbf{ReLoRA} \citep{lialin2024relora} is proposed to maintain the full rank of the model while still performing updates in a low-rank manner. Specifically, suppose we initialize the weights as $W=B_1A_1$ and train $B_1$ and $A_1$ for certain iterations. \cite{lialin2024relora} proposes to fix $B_1$ and $A_1$ and add new $B_2$ and $A_2$ to continue the training. After several such steps the training model becomes
\begin{equation*}
    W=W_0+B_{k}A_{k},\ W_0=B_1A_1+B_2A_2+\cdots+B_{k-1}A_{k-1}
\end{equation*}
where $W_0$ is already frozen. Therefore, ReLoRA gradually increases the rank of $W$, eventually reaching full rank; Partially inspired by ReLoRA, \textbf{SwitchLoRA} \citep{zhou2024revolutionizing} proposes pre-training with LoRA but periodically switches the columns/rows of the $B$/$A$ matrices, respectively, to achieve a higher rank update; Instead of parameterizing the update via low-rank matrices that still need to store a full-rank matrix, \textbf{SLTrain} \citep{han2024sltrain} undertakes a different route by directly parameterizing the weight with a sparse plus a low-rank matrix, decomposing the weight matrix in the following matter  
\begin{align}
    W=BA+S, \quad \mbox{\rm with $S$ being a sparse matrix}
\end{align}
 where we define $\delta:=nnz(S)/(mn)$ as the ratio of the non-zero entries. This low-rank and sparse decomposition allows one to approximate a full-rank matrix with a smaller number of parameters while being able to reduce the memory {utilized for weight, gradients and optimizer states} during pre-training. {SLTrain significantly boosts the model performance to match the full-rank training when adding a $10\%$ sparse matrix $S$ to the model, while still maintaining parameter and memory efficiency.} 

More works have proposed to consider alternative parameterization to further improve the pre-training efficiency. \textbf{LoQT} \citep{loeschcke2024loqt} similar to LORA, initializes two low-rank factors, $A$ and $B$, for each weight matrix $W$. However, only the matrix $B$ is trainable, and as $W$ and $A$ are not updated, they are quantized to optimize memory usage. Periodically, the product $AB$ is added back into the full-rank matrix $W$, i.e., 
$$W \leftarrow W+AB,$$
and the matrices $A$ and $B$ are reinitialized. \textbf{CoLA} \citep{liu2025cola} leverages the low-rank structure of a model's activations by factorizing the weight matrix $W$ as $AB$, and instead of applying the activation function on the matrix $AB$, it inserts it between the matrix factors. That is, the equation for an arbitrary layer takes the form $$h=A\sigma(Bx),$$ where $\sigma$ is a nonlinear activation. \textbf{AdaLoRA} \citep{zhang2023adalora} performs low-rank updates similar to LoRA but dynamically adjusts the rank of the update matrices based on their importance scores. Specifically, it parametrizes the update matrix $W_{0}$ as $W_{0}=P \Lambda Q$ to simulate the SVD and uses an importance metric to prune redundant singular values. Moreover, \citet{wei2024building} explored various structured parameterizations on the feedforward layers, such as low-rank and block-diagonal-diagonal matrices, along with a self-guided training scheme for effective and efficient training. \textbf{BTT} \citep{qiu2024compute} proposed a family of structured matrices through Kronecker product and investigated the use of different initialization and learning rate schedules. \textbf{CoMERA} \citep{yang2024comera} reshaped the weights into high-order tensors, which can be compactly represented using smaller tensor products; \cite{wang2024lora} have explored the use of a low-rank parameterization for pre-training and found 
that low-rank models are often not effectively optimized, leading to inefficiencies in their representation. Then, \textbf{LORO} \citep{mo2025parameter} proposed an approach that directly optimizes over the low-rank manifold in order to remove the redundancy during the pre-training;  

\subsection{Compression, Quantization and More}

Besides optimizer and model structure innovations, various work reduce the memory in model pre-training by infrastructure tricks such as compression, quantization, etc. {We stress such methods are not the focus of this paper, however for completeness we provide a brief overview of the related literature.}

Several works focus on compressing the information of the neural networks for memory efficient pre-training. \textbf{NeuZip} \citep{hao2024neuzip} proposes a weight compression method that leverages the entropy pattern of floating-point numbers in neural networks. It exploits the fact that the exponent bits of the weights have low entropy and therefore compresses them using a lossless compression algorithm, reducing memory consumption without sacrificing precision. During the forward pass and weight updates, the weight matrices of each layer remain in their compressed form until they are needed for computations, {i.e.,
the forward pass in the $k$ linear layer takes the following form
\begin{equation*}
\begin{aligned}
& \hat{W} \leftarrow \text{decompress} (c_{k})\\
& x_{k} \leftarrow \hat{W}x_{k-1} + b_{k}, 
\end{aligned}
\end{equation*}
where $c_{k}$ is the compressed form of the weight matrix $W_{k}$, $b_{k}$ is the bias and $x_{k-1}$ is the input to the $k$ layer.
}
\textbf{CompAct} \citep{shamshoum2024compact} introduces a method to reduce peak GPU memory usage by compressing a significant portion of the compute graph. Specifically, during the forward pass, CompAct compresses the activations using random projections and stores them. The compressed activations are then used in the backward pass to compute the gradients, ensuring that both the gradients and optimizer states remain low-rank. Finally, the compressed gradient is projected back to the original subspace during the weight update stage;
\textbf{VeLoRA} \citep{miles2024velora} is based on the observation that intermediate vectors generated during a forward pass (to be later used for gradient computation) can be heavily compressed without any performance loss. During the forward pass, VeLoRA splits tokens onto smaller sub-tokens and compresses them into a one-dimensional subspace. In the backward pass, these tokens are approximately reconstructed before applying the update rules.

\cite{chitsaz2024exploring} studies the effectiveness of linear quantization during pre-training. {Specifically, this study considers a linear quanitzation scheme that quantizes a vector $x$ in the following way:
\begin{equation*}
\begin{aligned}
& x_{\text{quant}} \leftarrow \text{clip}\left(\left\lfloor \frac{x}{s}\right\rceil -z ; N,P\right) \\
&\hat{x} = s(x_{\text{quant}}+z), 
\end{aligned}
\end{equation*}
where $\left\lfloor \cdot \right\rceil$ is the rounding operator, $s$ is the scaling factor, $N=-2^{b-1}, P=2^{b-1}-1$ describe the quantization range, $b$ is the bit width and $z$ is the offset.
}%
The authors observe memory savings without performance degradation when using 8-bit quantization for weights and activations. On the other hand, when quantization is applied to gradients or precision is reduced to 4 bits, either training instabilities arise or performance deteriorates. 
There are more works on the quantization of existing optimizers. \textbf{Q-Galore} \citep{zhang2024q} which explores more memory efficient GaLore via quantization; \textbf{4-bit Shampoo} \citep{wang20244} develops quantization methods for the preconditioners of second-order optimizer named Shampoo. {Finally, we emphasize that in the works mentioned above, the role of quantization is to reduce memory consumption during the pretraining phase. Therefore, these methods differ from quantization-based approaches designed to enhance performance during inference \citep{wei2025roste,xiao2023smoothquant}.}

Finally, another direction for memory efficient pre-training is to effectively reducing the amount of training data batch size. \textbf{CoLM} \citep{nguyen2024mini} proposes identifying small mini-batch coresets that mimic training with larger mini-batches to reduce high GPU memory demands. The key idea is to design coresets that capture the gradients of large random batches.

\section{Benchmarking {parameter and memory-efficient pre-training methods}}

Despite extensive works in this field, {to the best of our knowledge, there does not exist a comprehensive and thorough comparison among different algorithms. 
}
Therefore, {the core objective of this section is to develop a comprehensive benchmark of the most important parameter- and memory-efficient methods for LLM pre-training.}

\subsection{Methodology and benchmark settings}

{This section describes our benchmark's methodology and setting. We will focus on benchmarking representative methods that fall into either the category of memory-efficient optimizers or weight factorization. Note that again quantization and compression related approaches are not the focus of this work and we do not benchmark these approaches.
} 

\noindent\textbf{Memory-efficient Optimizers}. As discussed in the previous section, GaLore is pioneering the research on memory-efficient optimizers. 
{However, in general the methods that are} saving optimizer states will save the memory at the training/backward stage and not in the inference/forward stage. Besides, the projection of the gradient will naturally result in a loss of gradient information and (presumably) a worse convergence result. Surprisingly, many works in this direction, including Fira~\citep{chen2024fira} and Apollo~\citep{zhu2024apollo}, achieve better training results than the standard full-rank optimizer. 
{As mentioned in the previous section, Fira adds an additional orthogonal term in GaLore's projected gradient to restore the full-rankness of the update.
}
{Therefore,} we believe that the full-rankness is key for achieving a better performance for gradient-projection-related memory-efficient optimizers, and we thus include GaLore and Fira as our benchmarks.

\noindent\textbf{Weight Factorization}. As shown in Table \ref{table:related_works}, weight factorization approaches save memory in weights, gradients, and optimizer states. However, a fundamental question is whether or not low-rank factorization based methods (which is the starting point of all these works), such as directly decomposing the weight matrices as $W=BA$ or $W=W_0+BA$ (as in LoRA) can be used for pre-training. Due to a smaller number of parameters and the low-rank structure of the factorization, the capabilities of those approaches are limited. We therefore also include a method that restores the full rank of the weight matrices while still maintaining a small number of parameters compared to the full model, namely SLTrain~\citep{han2024sltrain}. Overall, we include low-rank, LoRA and SLTrain as our benchmarks. 

\noindent\textbf{Base performance of full-rank model pre-training}. Interestingly, in works such as \cite{chen2024fira,mo2025parameter}, the full model training is not outperforming the corresponding proposed parameter-efficient methods. We believe that this is due to the improper training of the baseline. In particular, methods such as SPAM~\citep{huang2025spam} and Stable-SPAM~\citep{huang2025stablespam} show that the baseline full-rank training performance could improve significantly with techniques such as momentum reset and adaptive gradient clipping. Adaptive gradient clipping dynamically adjusts the clipping threshold by tracking the maximum gradient magnitude observed over time, rather than relying on a pre-defined fixed value. Momentum reset periodically sets the momenta of the AdamW optimizers to zero to erase the influence of the batches with spiked gradients. We compare the training loss and evaluation perplexity of the plain AdamW optimizer and the optimizer with aforementioned techniques in Figure \ref{fig:loss_spike}, where we observe a superior performance after adding these techniques. We also report the perplexity with Stable-SPAM as the baseline full-rank performance in Table \ref{main_table_llama}.

\begin{figure}[t]
    \begin{center}
    \subfigure[Training Loss]{
        \includegraphics[width=0.445\textwidth]{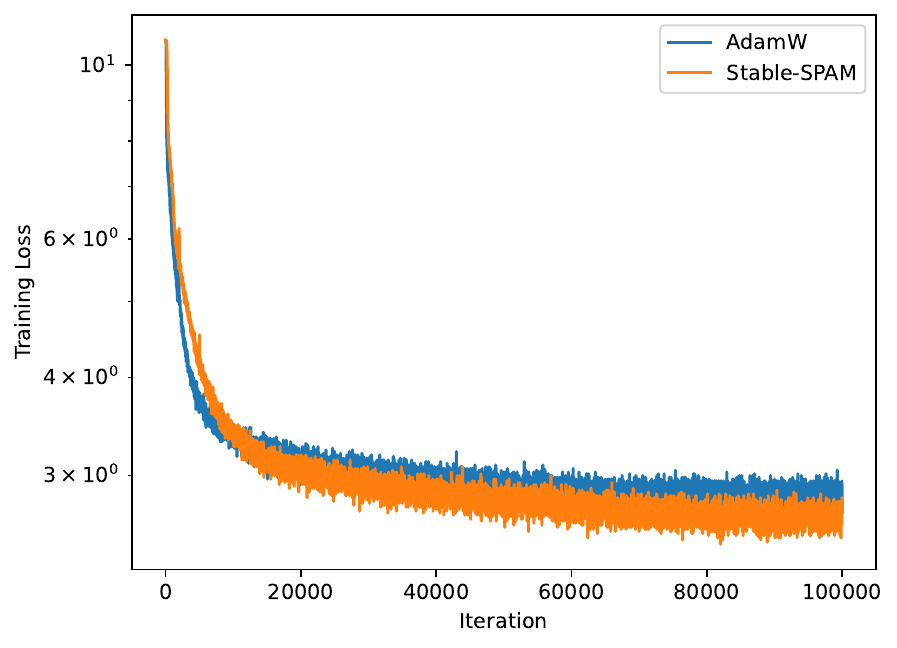}
    }
    \subfigure[Evaluation Perplexity]{
        \includegraphics[width=0.445\textwidth]{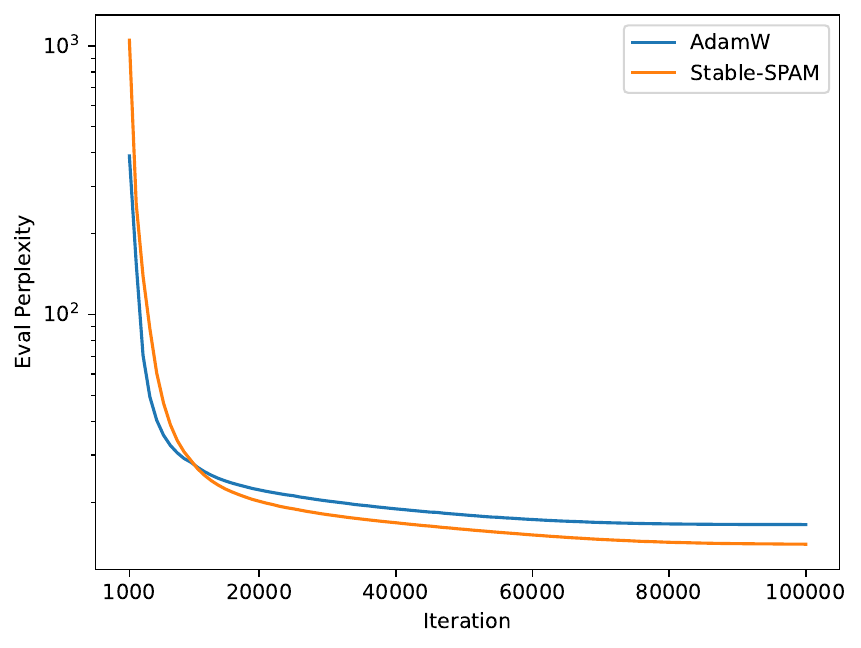}
    }
    \caption{Loss curve and evaluation perplexity before and after employing the momentum reset and adaptive gradient clipping, as suggested by \cite{huang2025stablespam}. 
    }
    \label{fig:loss_spike}
    \end{center}
\end{figure}

\noindent\textbf{Benchmark settings}. We follow the same experimental setting in both GaLore and SLTrain. Specifically, we consider LLaMA LLMs with pre-normalization, RMSnorm \citep{zhang2019root}, and SwiGLU activation \citep{shazeer2020glu}. 
We train LLaMA LLMs on C4 (Colossal Clean Crawled Corpus) dataset \citep{raffel2020exploring}, which is specially designed for pre-training. The training is performed without data repetition and we consider LLaMA with varying model sizes from 60M up to 1B parameters. 

As mentioned above, we include the standard \textbf{Full-rank}, \textbf{Low-rank}, \textbf{LoRA}, \textbf{SLTrain}, \textbf{GaLore} and \textbf{Fira} in our benchmark. We conducted a comprehensive hyperparameter search via wandb sweeps\footnote{See \url{https://docs.wandb.ai/guides/sweeps/} for details.} for all methods and report the best hyperparameter setting. We set the sequence length to 256 and the batch size to 512 for all experiments. For the full-rank method, we sweep the learning rates; For the low-rank methods LoRA, SLTrain, GaLore and Fira, we sweep the learning rates for different ranks.\footnote{For 60M model we tested ranks $r=32, 64, 128, 256$, while for larger models we tested $r=64, 128, 256$, as shown in Figure \ref{fig:sweep_results}.} 
{All methods are trained using the BF16 format.}
The corresponding memory consumption of each method tested is shown in Table \ref{table:memory_estimate} (which is also the data source of Figure \ref{fig:mem_category}).

\begin{table}
  \centering
\resizebox{0.95\textwidth}{!}{%
\begin{tabular}{c|cccc}
    \toprule
    Method & Weight & Activation & Optimizer & Gradient \\
    \midrule
    Full Model & $2mn$ & $2(m+n)b$ & $2*4mn$ & $4mn$ \\
    Low-rank & $2(mr+nr)$ & $2(m+n)b$ & $2*4(mr+nr)$ & $4(mr+nr)$   \\
    LoRA & $2(mn+mr+nr)$ & $2(m+n)b$ & $2*4(mr+nr)$ & $4(mr+nr)$  \\
    SLTrain \citep{han2024sltrain} & $2(mr+nr+\delta mn)$ & $2(m+n)b$ & $2*4(mr+nr+\delta mn)$ & $4(mr+nr+\delta mn)$ \\
    Galore \citep{zhao2024galore} & $2mn$ & $2(m+n)b$ & $4(mr+2nr)$ & $4mn$ \\
    Fira \citep{chen2024fira} & $2mn$ & $2(m+n)b$ & $4(mr+2nr+1)$ & $4mn$ \\    
    \bottomrule
  \end{tabular}
}
  \caption{\footnotesize Summary of memory (in bytes, 16-bits = 2-bytes) consumption for the tested methods, all using BF16 format and AdamW as the default optimizer. Here we assume that we have a weight matrix of dimension $m\times n$ (assume $m\leq n$), and an input batch size of $b$. For SLTrain, $\delta$ is the sparsity ratio. For the activation column, we do not assume the use of gradient and activation checkpointing.}
  \label{table:memory_estimate}
\end{table}

\noindent\textbf{Benchmark metrics}. Similarly to \cite{zhao2024galore,han2024sltrain}, we evaluate LLM pre-training performance using \textbf{perplexity} metric. Perplexity is mathematically defined as the exponential of the average log-likelihood of the predicted tokens in a sequence (i.e., the exponential of the validation loss). Hence, lower perplexity indicates more confident and accurate predictions of the next tokens. 

\subsection{Benchmarking Results}\label{sec:bench_res}

\begin{figure*}[!ht]
\begin{center}
\includegraphics[width=1\textwidth]{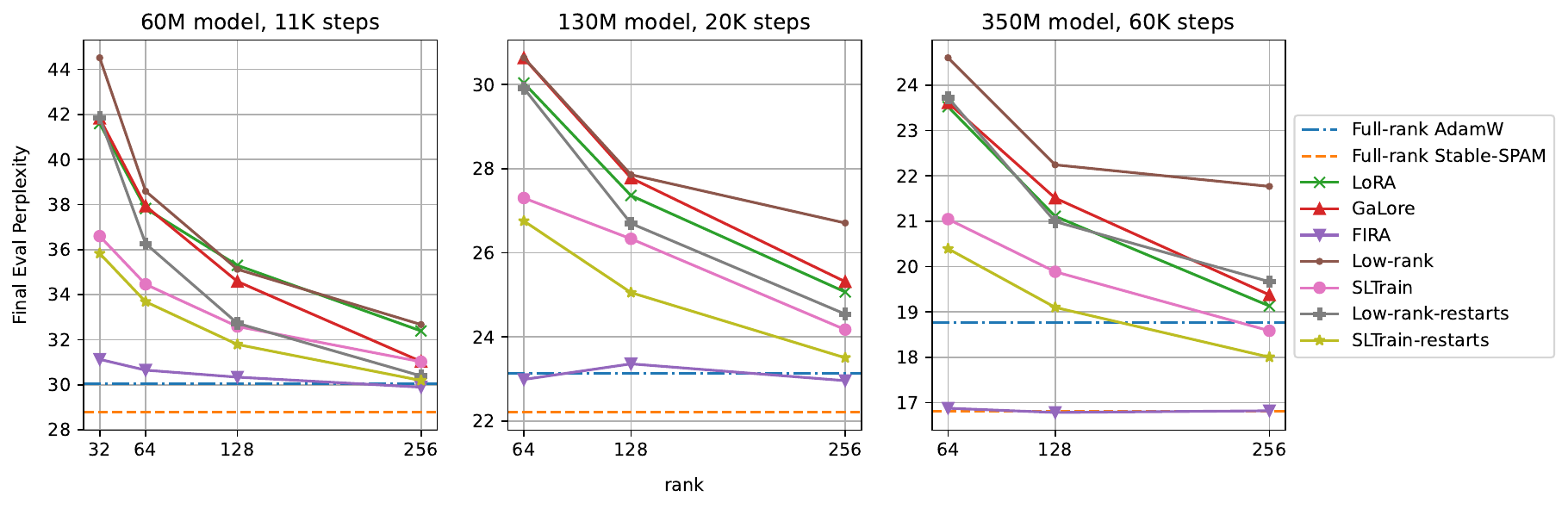}
\vspace{-10pt}
\caption{
Performance of 60M, 130M and 350M Llama models after hyperparameter search (Wandb sweep).  
}
\label{fig:sweep_results}
\end{center}
\end{figure*}

\begin{table}[t]
\centering
\renewcommand{\arraystretch}{1.3}
\setlength{\tabcolsep}{3pt}
\resizebox{0.95\textwidth}{!}{%
\begin{tabular}{l|ccc|ccc|ccc|ccc}
\toprule
  & \multicolumn{3}{c}{\textbf{60M}} & \multicolumn{3}{c}{\textbf{130M}} & \multicolumn{3}{c}{\textbf{350M}} & \multicolumn{3}{c}{\textbf{1B}} \\
  \midrule
  {$r$~/~$d$} & \multicolumn{3}{c|}{128 / 512} & \multicolumn{3}{c|}{256 / 768} & \multicolumn{3}{c|}{256 / 1024} & \multicolumn{3}{c}{512 / 2048} \\ 
{Tokens } & \multicolumn{3}{c|}{1.4B} & \multicolumn{3}{c|}{2.6B} & \multicolumn{3}{c|}{7.8B} & \multicolumn{3}{c}{13.1B} \\ 
\midrule 
& PPL & Param & Mem & PPL & Param & Mem & PPL & Param & Mem & PPL & Param & Mem \\
\midrule
\rowcolor{blue!15} {Full-Rank} (Adam) & 30.05 & 58 & 0.35  & 23.13 & 134 & 0.81 & 18.76 & 368 & 2.21 & 16.52 & 1339 & 8.04 \\ 
\rowcolor{blue!15} {Full-Rank} (Stable-SPAM) & 28.77 & 58 & 0.35  & 22.20 & 134 & 0.81 &  16.80 & 368 & 2.21 & 13.97 & 1339 & 8.04 \\ 
\midrule
\rowcolor{gray!15} {GaLore} \citep{zhao2024galore} & 34.58 & 58  & 0.28 & 25.31 & 134  & 0.61 & 19.37 & 368 & 1.59  & 15.57 & 1339 & 4.76 \\ 
\rowcolor{gray!15}{LoRA} 
& 35.30 & \bf{43} & 0.36  & 25.07 & \bf{94} & 0.84 & 19.13  & \bf{185} & 1.85  & 15.83 & \bf{609} & 6.34  \\ 
\rowcolor{gray!15}{Fira} \citep{chen2024fira} & \bf{30.34} & 58  & 0.28 & \bf{22.96} & 134  & 0.61 & \bf{16.82} & 368 & 1.59  & 15.10 & 1339 & 4.76 \\
\rowcolor{gray!15}{Low-Rank}   
& 35.13 & \bf{43} & \bf{0.24}  & 26.71 & \bf{94} & \bf{0.57} & 21.77  & \bf{185} & \bf{1.11}  & 18.22 & \bf{609} &\bf{3.66}  \\ 
\rowcolor{gray!15} {SLTrain} \citep{han2024sltrain} & 32.58 & 47 & 0.30 & 24.17 & 104 & 0.67 & 18.59 & 215 & 1.54 & 15.40 & 732 & 5.33 \\
\hline 
\rowcolor{green!15} {Low-Rank-restarts}    
& 32.73 & \bf{43} & \bf{0.24}  & 24.54 & \bf{94} & \bf{0.57} & 19.67 & \bf{185} & \bf{1.11}  & 15.01 & \bf{609} &\bf{3.66}  \\ 
\rowcolor{green!15} {SLTrain-restarts} & 31.79 & 47 & 0.30 & 23.50 & 104 & 0.67 & 18.00 & 215 & 1.54 & \bf{14.37} & 732 & 5.33 \\
\bottomrule
\end{tabular}
}
\caption{\footnotesize Validation perplexity (PPL($\downarrow$)), number of parameters in millions (Param), and estimated total memory cost in G (Mem). The last two methods are introduced in Section \ref{sec:innovation}. For full-rank, we report the best perplexity trainer by Stable-SPAM. For SLTrain, we fix $\delta = 0.1$. For the memory estimate, we follow \cite{han2024sltrain} and only include the memory of weights and optimizer states. See Appendix \ref{appendix:setting} for detailed hyperparameter settings of all experiments (We could not get a 14.31 perplexity for Fira 1B experiment as suggested in their paper \cite{chen2024fira}, also a 18.95 perplexity for GaLore 350M experiment as suggested in \cite{zhao2024galore}. The main cause of the former is that when we set the learning rate same as suggested in \cite{chen2024fira}, the training is not stable and diverges. The number we reported here are with a smaller learning rate.). The best-performing methods, excluding full-rank approaches (which naturally contain the optimal method in terms of PPL), are highlighted for each metric and model size.
}\label{main_table_llama} 
\end{table}

\begin{figure*}[!ht]
\begin{center}
    \includegraphics[width=0.6\textwidth]{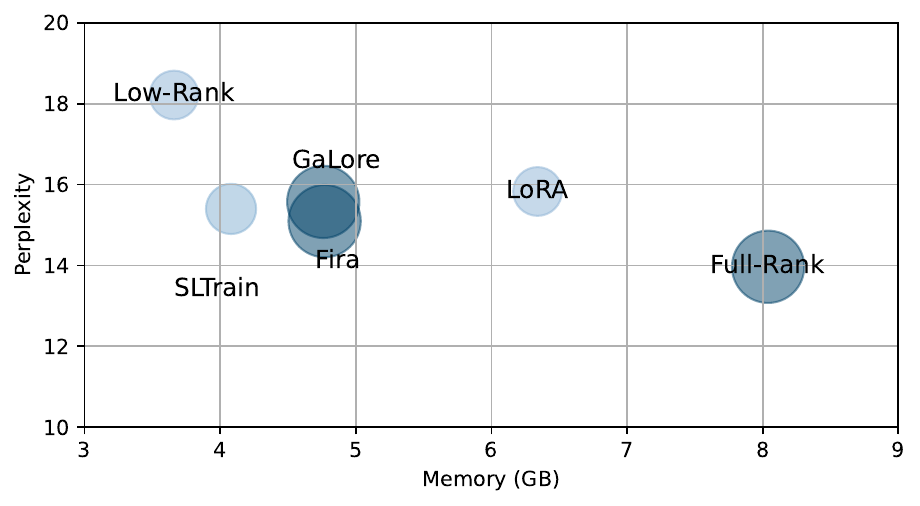}
\caption{\footnotesize Perplexity, memory, and parameter size for pretraining LLaMA 1B on the C4 dataset with different methods. The radius and color of each circle scale with parameter size. Overall, the methods which have smaller, lighter circles on the left bottom corner are desirable for pretraining.
}\label{fig:bubble_plot}
\end{center}
\end{figure*}

In this section, we present the benchmark results and discuss the main takeaways from the experiments. Figure \ref{fig:sweep_results} contains the hyperparameter search results for the 60M, 130M and the 350M models\footnote{Due to limited computational resources, we did not conduct a comprehensive hyperparameter search for the 1B model. Instead, we tested a few settings and report the optimal configuration.}. Several key results are summarized as follows.

\underline{First}, 
the \textit{full-rank model still delivers the best performance} and should be used as the baseline, provided that an appropriate optimizer and well-tuned hyperparameters are used, as in the case of Stable-SPAM~\citep{huang2025stablespam} method where gradient normalization and momentum reset techniques are implemented.

\underline{Second}, \textit{low-rank method achieves reasonably good performance}. In contrast to previous understanding  \citep{zhao2024galore}, where the low-rank method fails in pre-training (e.g. the perplexity is $142.53$ for the 1B models in \cite{zhao2024galore}), we achieve a perplexity of $18.22$ for the 1B model\footnote{We hypothesize that GaLore~\citep{zhao2024galore} initializes the low-rank matrix $B$ to zero which causes vanishing gradients and, consequently, the poor performance of the low-rank method. In this work we initialize $B$ with Kaiming Uniform \citep{he2015delving}, same as $A$.}. Plain low-rank model works particularly well for small models (60M, 130M), while the performance becomes less significant compared to other methods for larger models (350M, 1B). The performance drop in larger models may stem from unstable training dynamics when applying low-rank methods to larger base models~\citep{zhao2024galore,lialin2024relora}. 

\underline{Third}, \textit{restoring full-rankness will significantly boost the performance} of parameter-efficient pre-training, for both low-rank optimizers and low-rank weights. In particular, SLTrain yields better performance than the low-rank approach, and Fira yields better performance than GaLore, across all model sizes and rank settings. This indicates that incorporating high-rank updates is key to improving the performance of low-rank approaches.

\underline{Fourth}, \textit{memory-efficient optimizers and weight factorizations yield comparable performance}. For example SLTrain, GaLore and Fira achieve a perplexity of $15.40$, $15.57$ and $15.10$, respectively, for Llama 1B models (although none of them could match the $13.97$ performance of the full-rank model). We remark here that the performance of weight factorization methods (low-rank and SLTrain) could be further improved using the techniques we develop in Section \ref{sec:innovation}. Specifically, the low-rank and SLTrain methods can achieve a perplexity of $15.01$ and $14.37$, respectively, for the same 1B model, after the incorporation of the two proposed techniques (see ``Low-Rank-Restarts'' and ``SLTrain-Restarts'' in Table \ref{main_table_llama}).

\noindent\textbf{A scaling result for weight factorization methods}. In view of the hyperparameter search results in Figure \ref{fig:sweep_results}, we include a scaling law type result for full-rank (including all memory-efficient optimizers), low-rank and SLTrain methods in Figure \ref{fig:scaling}, where we plot the training curve in FLOPs (see Appendix \ref{appendix:setting} for the calculation of FLOPs) for each of the methods and different model sizes. 

From Figure \ref{fig:scaling}, we obtain the following conclusions: \underline{First}, all methods scale with respect to compute (in FLOPs) inverse proportionally (since both x and y-axis are in log scale); \underline{Second}, for different memory-efficient optimizers in full-rank model, their performance scales very differently (in terms of slopes), where GaLore appears to scale less efficiently due to the loss of the gradient information during the projection steps; \underline{Third}, the four categories (full-rank, low-rank, LoRA and SLTrain) scale similarly (in terms of slopes), indicating a scaling law type of observation: the final perplexity of the model primarily depends on the number of FLOPs, and is relatively independent of the specific model structures (full-rank or low-rank)  that we are using. 

\begin{figure*}[!ht]
\begin{center}
    \subfigure[Full-rank]{
        \includegraphics[width=0.445\textwidth]{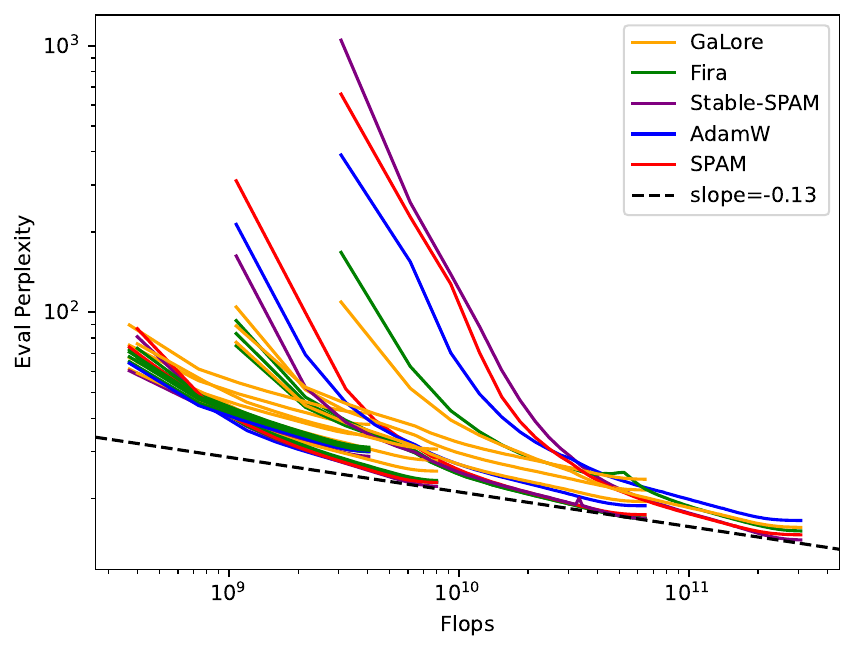}
    }
    \subfigure[Low-rank (AdamW)]{
        \includegraphics[width=0.445\textwidth]{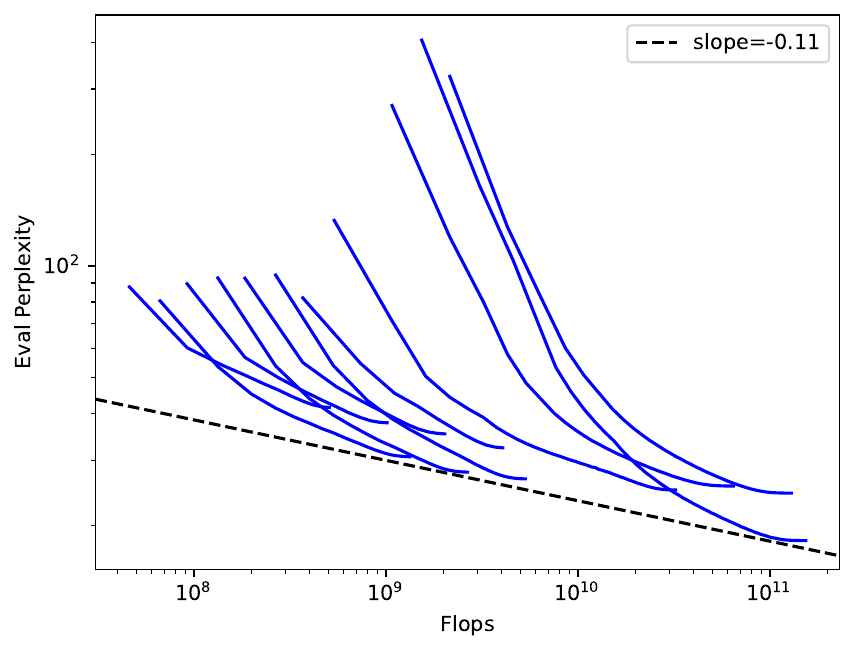}
    }
    \subfigure[LoRA (AdamW)]{
        \includegraphics[width=0.445\textwidth]{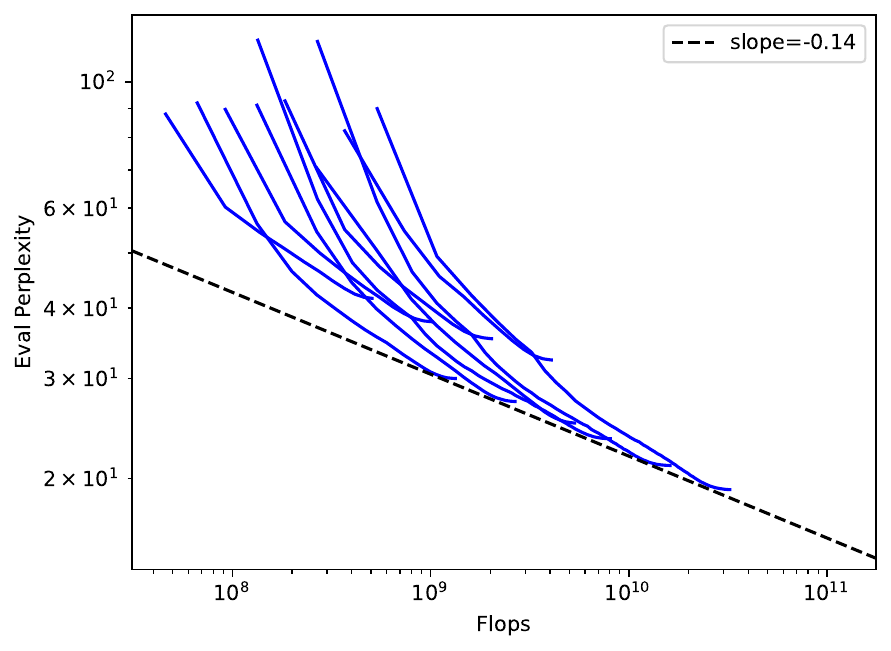}
    }
    \subfigure[SLTrain (AdamW)]{
        \includegraphics[width=0.445\textwidth]{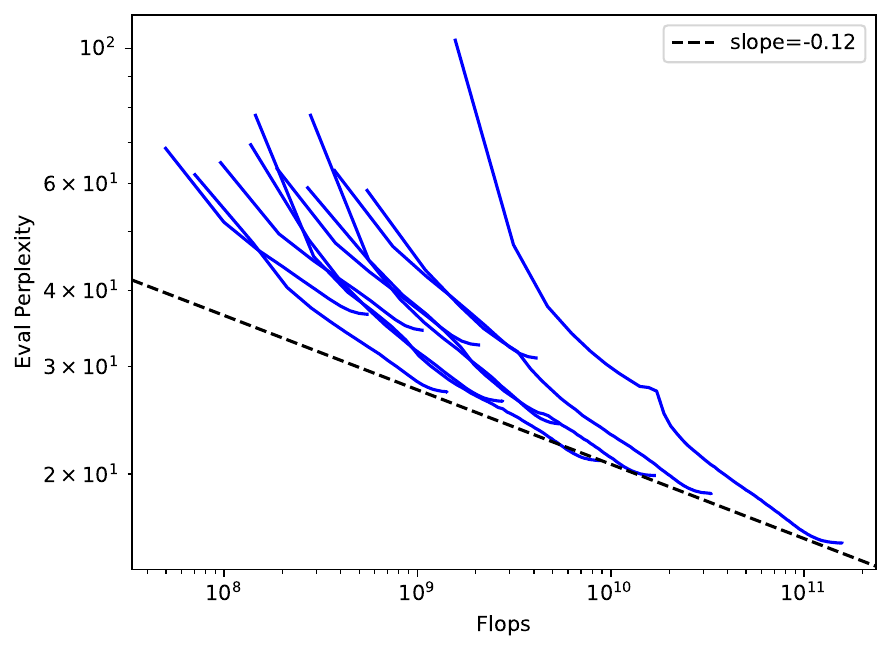}
    }
\caption{\footnotesize Scaling performance (in terms of evaluation perplexity) of different model structures with respect to compute (FLOPs). For the x-axis, the FLOPs are multiples of $10^{18}$. {Here we include all the training trajectories (with different the hyperparameter settings) to construct such scaling law.} 
}\label{fig:scaling}
\end{center}
\end{figure*}

\section{New techniques for parameter-efficient pretraining}\label{sec:innovation}

So far we have presented the main benchmark results and provided several key observations regarding efficient pre-training methods of both categories, namely memory-efficient optimizers and weight factorizations. One major conclusion from the previous section is that \textit{full-rank training with a better optimizer still outperforms all low-rank-based methods via weight factorization}. In this section, we discuss several key innovations for low-rank-type weight factorizations that aim to improve performance while still saving more memory than most memory-efficient optimizers (see Figure \ref{fig:mem_category} and Table \ref{table:related_works}).

\subsection{Weight Refactorization}
The first method we propose is the \textbf{weight refactorization}, a technique designed to improve the performance of weight factorization-based approaches. Consider training the LLM via the low-rank factorization, $W:=BA$, where we try to optimize the following problem 
\begin{equation*}
    \min_{A, B} \left\{ L(B, A) \coloneqq \ell(BA) \right\}
\end{equation*}
with $\ell$ being the loss function, $B\in\RR^{m\times r}$ and $A\in\RR^{r\times n}$. A natural observation is that, if $W=BA$ is fixed, then the loss function and the entire forward pass will remain the same, regardless of the specific value of $B$ and $A$. 
This allows us to refactor $W$ as $W=B'A'$ without changing model output, where in general $A' \neq A, B' \neq B$, to further benefit pre-training by promoting more efficient and stable training dynamics. In particular, let us denote the singular value decomposition of $W$ as $W=U\Sigma V^\top$. Then a natural refactorization can be obtained with the following assignment 
\begin{equation}\label{eq:weight_refactor}
    B'=U\sqrt{\Sigma}\text{ and }A'=\sqrt{\Sigma} V^\top.
\end{equation}

Now we examine the refactorization \eqref{eq:weight_refactor} theoretically, where we will show that the refactorization will lead to better conditioning for the algorithm. To this end, we examine the condition number of the Hessian $\nabla^2 L(B, A)$ at (local) optimality in the following lemma, whose proof is given in Appendix \ref{appendix:proof}. 

\begin{lemma}\label{lemma_refactor}
Let $B^* A^*$ be a local minimizer of $\ell$, i.e. $\nabla \ell (B^*A^*) = 0$ and $\nabla^2 \ell (B^* A^*) \succ 0$. Let $\underline{\lambda}, \overline{\lambda}$ be the smallest and largest eigenvalue of $\nabla^2 \ell (B^* A^*)$, and thus $\overline{\lambda} \geq \underline{\lambda} > 0$. Then we can show 
\begin{align*}
     \frac{\underline{\lambda} (\lambda_{\rm max}(A^{*\top} A^*) + \lambda_{\rm max}(B^* B^{*\top}))}{\overline{\lambda} (\lambda_{\rm min}(A^{*\top} A^*) + \lambda_{\rm min}(B^* B^{*\top}))} \leq \kappa\big( \nabla^2 L(B^*, A^*) \big) \leq  \frac{\overline{\lambda} (\lambda_{\rm max}(A^{*\top} A^*) + \lambda_{\rm max}(B^* B^{*\top}))}{\underline{\lambda} (\lambda_{\rm min}(A^{*\top} A^*) + \lambda_{\rm min}(B^* B^{*\top}))}.
\end{align*}
Let $B^* A^* = U \Sigma V^\top$ be the SVD and suppose $B^* = U \Sigma^{\alpha}$ and $A^* = \Sigma^{1-\alpha} V^\top$, then $\frac{ \lambda_{\rm max}(A^{*\top} A^*) + \lambda_{\rm max}(B^* B^{*\top})}{ \lambda_{\rm min}(A^{*\top} A^*) + \lambda_{\rm min}(B^* B^{*\top})}$ is minimized when $\alpha = 1/2$. 
\end{lemma}

Intuitively, Lemma \ref{lemma_refactor} suggests a balanced distribution of eigenvalues among $B, A$ as in \eqref{eq:weight_refactor} improves the conditioning of the Hessian at optimality. This translates to a faster local convergence according to standard convergence guarantees \cite[Chapter 1]{bertsekas1997nonlinear} since the local rate directly depends on the local condition number.

In our experiments, we refactor the weights for low-rank and SLTrain every 200 update iterations, which consistently improves performance under our experimental settings.

\subsection{Momentum Reset}
The second innovation we examine is the \textbf{momentum reset}, which we also employed for the baseline full-rank training in Stable-SPAM~\citep{huang2025stablespam}. This technique periodically resets the momentum of the AdamW optimizer to zero. In this section, we provide a preliminary theoretical analysis of how this mechanism benefits training, an aspect not covered which is not provided in \cite{huang2025stablespam}. 

Consider the finite-sum full-rank model training:
\begin{equation}\label{eq:finite_sum}
    \min_{W}\ \ell(W)=\sum_{i=1}^{n}\ell_i(W).
\end{equation}
For simplicity, we will discuss the application of momentum reset to a simpler algorithm, namely stochastic gradient descent with momentum (SGD-M):
\begin{equation}\label{eq:sgdm}
    \begin{aligned}
    M^t & =\beta_t M^{t-1}+\nabla \ell_{i_t}\left(W^t\right), \\
    W^{t+1} & =W^t-\gamma_t M^t,
    \end{aligned}
\end{equation}
where $i_t$ is randomly picked in $[1,2,...,n]$ and $M^{-1}=0$. For update \eqref{eq:sgdm}, 
we have the following known result:
\begin{theorem}[Theorem 7.4 in \cite{garrigos2023handbook}]\label{thm:old_sgdm}
    Suppose that for all $i$, the loss $\ell_{i}$ in \eqref{eq:finite_sum} is convex, $L$-smooth and it holds that $\E_i\|\nabla \ell_i\|^2\leq G^2$. SGD-M (the update steps are given in \eqref{eq:sgdm}) with 
    \begin{equation*}
        \gamma_t=\frac{2\eta}{t+3},\ \beta_t=\frac{t}{t+2},\text{ with }\eta\leq\frac{1}{4L}
    \end{equation*}
    achieves the following convergence:
    \begin{equation*}
        \E[f(W^{T})-f^*]\leq \frac{\|W^0-W^*\|^2}{\eta (T+1)}+2\eta G^2.
    \end{equation*}
\end{theorem}

For SGD-M with momentum reset (SGD-M-R, Algorithm \ref{algo:sgdmr}), where we reset the momentum to zero after every $T$ iterations, we have the following new result (proof available in Section \ref{appendix:proof}):

\begin{algorithm}[t] 
\caption{SGD-M with momentum reset (SGD-M-R)} 
\begin{algorithmic}
\STATE {\bfseries Input:} Initialize weight $W^0$ and $t=0$

\FOR{$k=0,1,\ldots, K-1$}
\STATE Set $M^{t-1} = 0$ \\
\FOR{$j=0,1,...,T-1$}
    \STATE Sample $i_t$ from $i=1,2,...,n$ \\
    \STATE $M^{t}=\beta_t M^{t-1}+\nabla \ell_{i_t}\left(W^t\right)$ \\
    \STATE $W^{t+1}=W^t-\gamma_t M^t$ \\
    \STATE $t=t+1$
\ENDFOR
\ENDFOR
\end{algorithmic}
\label{algo:sgdmr}
\end{algorithm}

\begin{theorem}\label{thm:sgdm_restart}
    Suppose the assumptions of Theorem \ref{thm:old_sgdm} hold. Algorithm \ref{algo:sgdmr} with 
    \begin{equation*}
        \gamma_t=\frac{2\eta}{t+3},\ \beta_t=\frac{t}{t+2},\text{ with }\eta\leq\frac{1}{4L},
    \end{equation*}
    achieves the following convergence:
    \begin{equation*}
        \E[f(W^{KT}) - f^*]\leq \sum_{k=0}^{K-1}\frac{\|W^{kT} - W^*\|^2}{\eta(k+1) (T+1)} + \sum_{k=0}^{K-1}\frac{2\eta G^2}{k+1}.
    \end{equation*}
\end{theorem}

In summary, running both SGD-M \eqref{eq:sgdm} and its momentum rest version (Algorithm \ref{algo:sgdmr}) for $KT$ iterations, their convergence behaviors can be characterized by the following:
\begin{align*}
    \E[f(W^{KT})-f^*]&\leq \frac{\|W^0-W^*\|^2}{\eta (KT + 1)}+2\eta G^2   & {\mbox{rate for SGD-M}}\\
    \E[f(W^{KT}) - f^*]&\leq \sum_{k=0}^{K-1}\frac{\|W^{kT} - W^*\|^2}{\eta(k+1) (T+1)} + \sum_{k=0}^{K-1}\frac{2\eta G^2}{k+1}\quad & {\mbox{rate for SGD-M-R}}.
\end{align*}

Note that if $\|W^{kT}-W^*\|/\|W^{0}-W^*\|\approx1$ remains a constant order, then the latter is converging slower than the former (by a logarithm factor). However if $\|W^{kT}-W^*\|/\|W^{0}-W^*\|=\Omega(\sqrt{1/(kT)})$, then the latter becomes 
\begin{align*}
    f(\hat{W})-f^*\leq \sum_{k=1}^{K}\frac{\|W^0-W^*\|^2}{\eta k^2 T^2}+2\eta G^2\log(K)
\end{align*}
which is significantly faster than the former, since usually $T$ is much larger than $K$ (the number of momentum resets) in practice.

In conclusion, if the loss landscape is mild (e.g. (strongly) convex or in an overparametrized regime so that the iterate is indeed converging to the optimal point), momentum resetting could boost the performance significantly. Analyzing the benefit of momentum reset for Adam is complicated and is therefore left as a direction for future work. In our experiments, we found that resetting the momentum every 200 updates yields the best performance, and we therefore adopt this choice of hyperparameter.

\subsection{Experiment Results}

The two techniques proposed above can be applied to the low-rank and SLTrain methods, resulting in novel training strategies. In fact, they can be implemented in a straightforward manner. Specifically, we can apply the refactorization technique to the low-rank matrix factors $B$ and $A$ present in both the low-rank and the SLTrain methods. Moreover, we can apply momentum reset to the optimizer involved in the low-rank and the SLTrain methods. 

The results are included in Figure \ref{fig:sweep_results} and Table \ref{main_table_llama}. The low-rank and the SLTrain methods with the addition of the two new techniques are referred to as ``Low-rank restarts'' and ``SLTrain-restarts'', respectively. From the figures and table we draw the conclusion that the two innovations greatly boost the performance of weight factorization methods. Notably, low-rank-restarts for 1B pre-training yields a comparable perplexity (15.013) to GaLore and Fira (15.573 and 15.098, respectively) when applied to the same model, while also reducing memory usage compared to these memory-efficient optimizers (i.e., GaLore and Fira).

\begin{figure}[!ht]
    \begin{center}
    \subfigure[60m]{
        \includegraphics[width=0.47\textwidth]{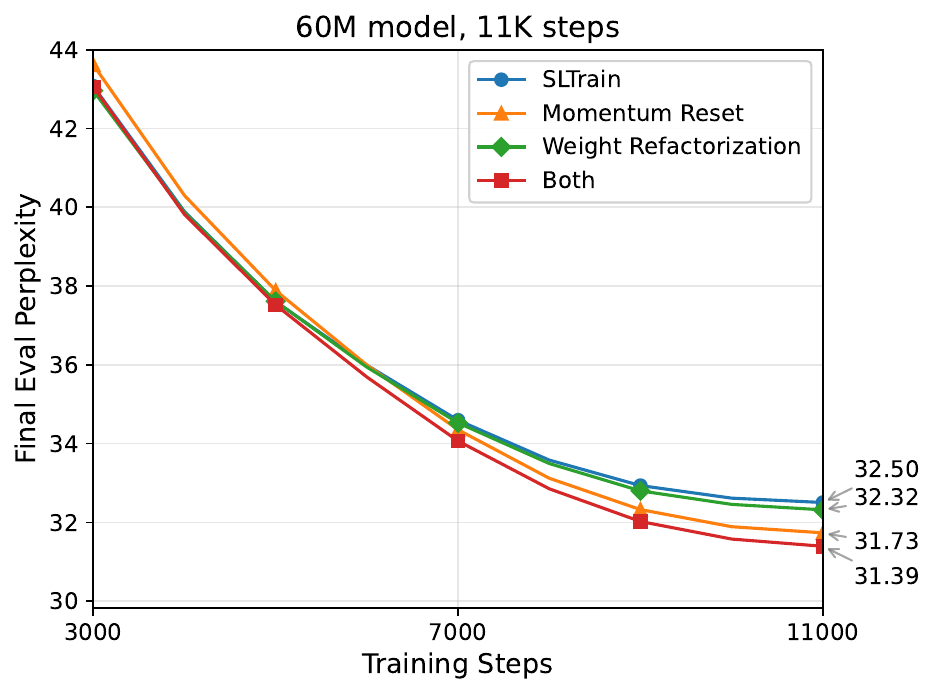}
    }
    \subfigure[130m]{
        \includegraphics[width=0.47\textwidth]{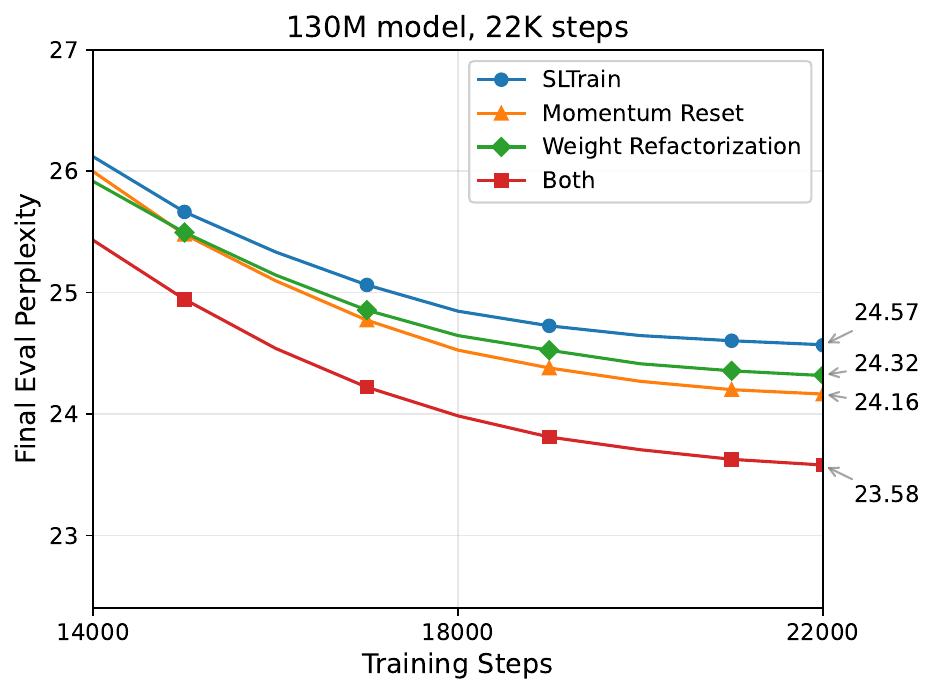}
    }
    \caption{\footnotesize Ablation study on the effect of the two proposed techniques. Note that here we use the same learning rate $0.003$ on both methods, differently than the case where a hyperparameter search is performed (e.g., in Table \ref{main_table_llama}). 
    }
    \label{fig:ablation}
    \end{center}
\end{figure}

\noindent\textbf{Ablation Study}. A natural question is the following: which of these two proposed techniques contributes more to the superior performance of weight factorization methods? To answer it we conduct an ablation study to explore the individual effects of weight refactorization and momentum reset. Specifically, we run SLTrain on the Llama 60M and 130M models, with and without the application of the two techniques, resulting in four evaluation curves depicted in Figure \ref{fig:ablation}. Figure \ref{fig:ablation} indicates that momentum reset contributes more towards a faster convergence, whereas weight refactorization could also steadily improve the performance of weight refactorization methods.

\section{Conclusions}
{In this work, we explored the potential of parameter- and memory-efficient methods, such as weight factorization and gradient compression, when used in LLM pre-training. Our primary motivation was to determine whether such methods can enhance pre-training efficiency while maintaining performance comparable to full-model training. Additionally, we aimed to identify strategies to further narrow the performance gap. To this end, we conducted a benchmark evaluation of several representative efficient pre-training approaches, which allowed us to identify the key factors that contribute to both efficiency and high performance in pre-training. We also proposed two practical techniques, weight refactorization and momentum reset, to further enhance the performance of efficient pre-training methods. In future work, we aim to expand our experiments to include more models, datasets, and tasks, as there are several important models and datasets that we did not consider. We also plan to explore additional efficient pre-training techniques.
}

\section{Acknowledgements}

The authors thank Prof. Caiwen Ding and Yuebo Luo for the helpful discussion.

\newpage
\bibliographystyle{abbrvnat} 
\bibliography{reference}

\section{Appendix}

\subsection{Memory Calculation: Model, Gradient, Optimizer States, and Activations}

In this section, we discuss how the memory estimations are calculated for all methods and model sizes. The total memory requirements for a deep learning model can be broken down into four key components:

\begin{itemize}
    \item \textbf{Model Parameters:} The memory required to store the model's weights.
    \item \textbf{Gradients:} The memory needed to store gradients during backpropagation.
    \item \textbf{Optimizer States:} Memory required for optimizer-specific variables (e.g., Adam uses first and second moment estimates).
    \item \textbf{Activations:} Memory required to store intermediate results during forward propagation.
\end{itemize}

The total memory requirement can be expressed as:

\[
\text{Memory} = \text{Model} + \text{Gradient} + \text{Optimizer States} + \text{Activations}.
\]

\subsubsection{Memory for Model, Gradients, and Optimizer States}

\begin{itemize}
    \item \textbf{Model Parameters:} The memory required is proportional to the number of parameters in the model and the storage format. For BF16 (2 bytes per parameter), the memory is \( \text{Number of Parameters} \times 2 \).
    \item \textbf{Gradients:} The memory required is equal to that of the model parameters, as each parameter has a corresponding gradient.
    \item \textbf{Optimizer States:} For Adam, the optimizer requires memory for two additional states per parameter (first and second moments), which is twice the gradient memory.
\end{itemize}

\subsubsection{Activation Memory}

Activation memory depends on the architecture and specific computations performed during forward propagation. For the LLaMA architecture, activation memory can be decomposed as follows:

\paragraph{Definitions of Parameters:}
\begin{itemize}
    \item \textbf{\( b \)}: batch size
    \item \textbf{\( s \)}: sequence length
    \item \textbf{\( h \)}: embedding and hidden dimension
    \item \textbf{\( l \)}: number of Transformer layers
    \item \textbf{\( a \)}: number of attention heads
    \item \textbf{\( k \)}: output dimension of \( W_{\text{gate}} \) and \( W_{\text{up}} \) in the MLP
    \item \textbf{\( v \)}: vocabulary size
\end{itemize}

\paragraph{Activation Memory Components:}
\begin{itemize}
    \item \textbf{Embedding Activations:} \( sh \), for storing input embeddings.
    \item \textbf{Multi-Head Attention Activations:} \( 5sh + 2s^2a \), including intermediate computations for queries, keys, values, attention outputs, and softmax.
    \item \textbf{MLP Activations:} \( 4sk \), for gate, activation, and intermediate values in the MLP.
    \item \textbf{Probability (Output) Activations:} \( 2sv \), for logits and softmax probabilities.
\end{itemize}

\paragraph{MLP Activation Memory Derivation:}

The MLP layer in Transformer architecture typically contains two linear transformations and a non-linear activation. During forward propagation, the following intermediate activations are stored:

\begin{itemize}
    \item \( \mathbf{g} = W_{\text{gate}} \mathbf{x} \): Output of the gating branch.
    \item \( \mathbf{u} = W_{\text{up}} \mathbf{x} \): Intermediate output of the signal branch before activation.
    \item \( \mathbf{a} = \operatorname{SiLU}(\mathbf{u}) \): Activated output of the signal branch.
    \item \( \mathbf{z} = \mathbf{g} \odot \mathbf{a} \): Element-wise product of the gating and activated signal branches.
\end{itemize}

To compute gradients during backpropagation, we need to store the following:

1. \( \mathbf{g} \), to compute the gradient of \( W_{\text{gate}} \).
2. \( \mathbf{u} \), to compute the gradient of \( W_{\text{up}} \), as SiLU's derivative depends on \( \mathbf{u} \).
3. \( \mathbf{a} \), as part of the gradient computation for \( \mathbf{z} \).
4. \( \mathbf{z} \), or equivalently \( \mathbf{g} \) and \( \mathbf{a} \), to reconstruct \( \mathbf{z} \).

Thus, the total memory required for MLP activations is:

\[
\text{MLP Activations} = 4sk,
\]

where \( sk \) corresponds to the memory for each activation (e.g., \( \mathbf{g}, \mathbf{u}, \mathbf{a}, \mathbf{z} \)).

\paragraph{Total Activation Memory:}
The total activation memory is given by:

\begin{equation}
\text{Total Activations} = b \cdot \Big( sh + l \big( 5sh + 2s^2a + 4sk \big) + 2sv \Big) \tag{1}
\end{equation}

For large models (when \( l \) dominates), this can be simplified to:

\begin{equation}
\text{Total Activations (Simplified)} = bl \cdot \big( 5sh + 2s^2a + 4sk \big) \tag{2}
\end{equation}

\subsubsection{Example: LLaMA 7B Model}

We illustrate the memory calculation using the \textbf{LLaMA 7B} model. The relevant parameters are:

\begin{itemize}
    \item \textbf{Batch size (\( b \))}: 1
    \item \textbf{Sequence length (\( s \))}: 2048
    \item \textbf{Embedding/Hidden size (\( h \))}: 4096
    \item \textbf{Number of Transformer layers (\( l \))}: 32
    \item \textbf{Attention heads (\( a \))}: 32
    \item \textbf{Intermediate MLP size (\( k \))}: 11008
    \item \textbf{Vocabulary size (\( v \))}: 32000
\end{itemize}

Assume \textbf{1GB = \( 1024^3 \)} bytes, and each parameter uses 2 bytes (BF16 format).

\paragraph{1. Model Parameters:}
The memory for storing model parameters is:

\[
\text{Model Memory} = 7,000,000,000 \times 2 = 14,000,000,000 \, \text{bytes} \approx 13.04 \, \text{GB}.
\]

\paragraph{2. Gradients:}
The memory for gradients is identical to the model memory:

\[
\text{Gradient Memory} = 14,000,000,000 \, \text{bytes} \approx 13.04 \, \text{GB}.
\]

\paragraph{3. Optimizer States:}
For Adam, the optimizer requires twice the gradient memory:

\[
\text{Optimizer Memory} = 2 \times 14,000,000,000 = 28,000,000,000 \, \text{bytes} \approx 26.08 \, \text{GB}.
\]

\paragraph{4. Activations:}
The activation memory is computed as follows:

\begin{itemize}
    \item \textbf{Embedding Activations:} \( sh = 2048 \times 4096 = 8,388,608 \, \text{bytes} \).
    \item \textbf{Attention Activations:}
    \begin{align*}
        5sh &= 5 \times 2048 \times 4096 = 41,943,040, \\
        2s^2a &= 2 \times 2048^2 \times 32 = 268,435,456, \\
        \text{Total Attention:} &\quad 32 \cdot (41,943,040 + 268,435,456) = 9,932,111,872 \, \text{bytes}.
    \end{align*}
    \item \textbf{MLP Activations:} \( l \cdot 4sk = 32 \cdot 4 \cdot 2048 \cdot 11008 = 2,885,681,152 \, \text{bytes}. \)
    \item \textbf{Probability Activations:} \( 2sv = 2 \times 2048 \times 32000 = 131,072,000 \, \text{bytes}. \)
\end{itemize}

The total activation memory is:

\[
\begin{split}
    \text{Total Activations} &= 8,388,608 + 9,932,111,872 + 2,885,681,152 + 131,072,000 \\ &= 12,957,253,632 \, \text{bytes} \approx 12.96 \, \text{GB}.
\end{split}
\]

\paragraph{Total Memory Requirements:}

The total memory required for the LLaMA 7B model is:

\begin{itemize}
    \item \textbf{Model Parameters:} 13.04 GB
    \item \textbf{Gradients:} 13.04 GB
    \item \textbf{Optimizer States:} 26.08 GB
    \item \textbf{Activations:} 24.13 GB
\end{itemize}

Therefore the total memory is calculated as:
\[
\text{Total Memory} = 13.04 + 13.04 + 26.08 + 24.13 = 76.29 \, \text{GB}.
\]

Other methods and model sizes can be calculated similarly.

\subsection{Detailed settings for the benchmark experiments}\label{appendix:setting}

For models sizes up to 350M parameters, for all the methods we choose the best learning rate based on validation perplexity from \{0.0005, 0.001, 0.002, 0.003, 0.005, 0.01\}. For the 1B model we use as a starting point the optimal learning rate from the 350M model and further manually tune around that value. The learning rate schedule for all experiments is the same as in \cite{zhao2024galore}, linear warm up for the first 10\% of the training iterations, followed by cosine annealing decaying to 10\% of the initial learning rate. In addition, for the proposed methods on every restart we drop the learning rate to zero and perform a quick linear warm up, similar to \cite{lialin2024relora}.   For GaLore, Fira, SLTrain and Stable-SPAM all other hyperparameters are chosen according to their official code-bases.

\noindent\textbf{Computation of FLOPs}. We use the following empirical law for compute~\citep{kaplan2020scaling}: $C=6NBS$ , where $C$ is the compute (in FLOPs), $N$ is the number of non-embedding parameters, $B$ is the batch size (both of which we know in advance) and $S$ is the number of training steps. Then by logging ($S$, Perplexity) pairs during execution we can generate a compute-perplexity plot.

\subsection{Proofs for Section \ref{sec:innovation}}\label{appendix:proof}

\begin{proof}[Proof of Lemma \ref{lemma_refactor}]
We first derive $\nabla^2 L(B^*, A^*)$ in terms of $\nabla^2 \ell(B^*A^*)$ as 
\begin{align*}
    &\langle [\Delta B, \Delta A], \nabla^2 L(B^*, A^*)[\Delta B, \Delta A] \rangle \\
    &= \langle \Delta B A^* + B^* \Delta A , \nabla^2 \ell (B^* A^*) [\Delta BA^* + B^* \Delta A] \rangle \\
    &= \mathrm{vec}([\Delta B, \Delta A])^\top \begin{bmatrix} A^* \otimes I \\ I \otimes B^{*\top} \end{bmatrix} \mathrm{mat}( \nabla^2 \ell(BA)  ) \begin{bmatrix} 
        A^{*\top} \otimes I  & I \otimes B^*
    \end{bmatrix} \mathrm{vec}([\Delta B, \Delta A])
\end{align*}
where we use $\mathrm{vec}(\cdot), \mathrm{mat}(\cdot)$ to represent vectorization and matricization operations. Based on the derivation, the eigenvalues of $\nabla^2 L(B^*, A^*)$ are the eigenvalues of the matrix 
\begin{align*}
    \begin{bmatrix} A^* \otimes I \\ I \otimes B^{*\top} \end{bmatrix} \mathrm{mat}( \nabla^2 \ell(B^*A^*)  ) \begin{bmatrix} 
        A^{*\top} \otimes I  & I \otimes B^*
    \end{bmatrix}
\end{align*}
which are the eigenvalues of 
\begin{align*}
    H^* \coloneqq  \mathrm{mat}(\nabla^2 \ell(B^*A^*) )^{1/2} ( (A^{*\top} A^* \otimes I) + (I \otimes B^*B^{*\top})  ) \mathrm{mat}(\nabla^2 \ell(B^*A^*) )^{1/2}
\end{align*}
From \citep{merikoski2004inequalities}, we can both lower and upper bound the eigenvalues of $H^*$ as
\begin{align*}
    &\underline{\lambda} (\lambda_{\rm max}(A^{*\top} A^*) + \lambda_{\rm max}(B^* B^{*\top})) \leq \lambda_{\rm max}(H^*) \leq \overline{\lambda} (\lambda_{\rm max}(A^{*\top} A^*) + \lambda_{\rm max}(B^* B^{*\top})) \\
    &\underline{\lambda} (\lambda_{\rm min}(A^{*\top} A^*) + \lambda_{\rm min}(B^* B^{*\top})) \leq \lambda_{\rm min}(H^*)\leq \overline{\lambda}(\lambda_{\rm min}(A^{*\top} A^*) + \lambda_{\rm min}(B^* B^{*\top}))
\end{align*}
and thus the condition number can be bounded as 
\begin{align*}
    \frac{\underline{\lambda} (\lambda_{\rm max}(A^{*\top} A^*) + \lambda_{\rm max}(B^* B^{*\top}))}{\overline{\lambda} (\lambda_{\rm min}(A^{*\top} A^*) + \lambda_{\rm min}(B^* B^{*\top}))} \leq \kappa(H^*) \leq \frac{\overline{\lambda} (\lambda_{\rm max}(A^{*\top} A^*) + \lambda_{\rm max}(B^* B^{*\top}))}{\underline{\lambda} (\lambda_{\rm min}(A^{*\top} A^*) + \lambda_{\rm min}(B^* B^{*\top}))}
\end{align*}
It remains to bound the term $\frac{ \lambda_{\rm max}(A^{*\top} A^*) + \lambda_{\rm max}(B^* B^{*\top})}{ \lambda_{\rm min}(A^{*\top} A^*) + \lambda_{\rm min}(B^* B^{*\top})}$. 

Let $W^* = B^* A^*$ and suppose $B^* = U \Sigma^{\alpha}$ and $A^* = \Sigma^{1-\alpha} V^\top$ where $U \Sigma V^\top = W^*$ is the SVD. Then we can further derive 
\begin{align*}
    \frac{ \lambda_{\rm max}(A^{*\top} A^*) + \lambda_{\rm max}(B^* B^{*\top})}{ \lambda_{\rm min}(A^{*\top} A^*) + \lambda_{\rm min}(B^* B^{*\top})} &= \frac{\max(\Sigma)^{2\alpha} + \max(\Sigma)^{2-2\alpha}}{\min(\Sigma)^{2\alpha} + \min(\Sigma)^{2-2\alpha}} 
\end{align*}
which is minimized when $\alpha = 1/2$. 
\end{proof}

\begin{proof}[Proof of Theorem \ref{thm:sgdm_restart}]
    We follow the proof in \citet[Section 7]{garrigos2023handbook}.
    
    Suppose we consider the Algorithm \ref{algo:sgdmr} in the iteration $t=kT, kT+1,...,(k+1)T$ (corresponding to $j=0,1,...,T-1$) where $k\in\{0,1,2,3...,K-1\}$ is fixed. Since we reset the momentum $M^t$ into zero at $t=kT-1$, we are essentially running \eqref{eq:sgdm} for $T$ steps, with
    \begin{equation*}
        \gamma_j=\frac{2\eta}{N+j+3},\ \beta_j=\frac{N+j}{N+j+2},\text{ with }\eta\leq\frac{1}{4L},\ j=0,1,...,T-1
    \end{equation*}
    where $N=kT$. 

    Denote $Z^{-1} = W^{N}$, $\lambda_j=\frac{N+j}{2}$ and $Z^{j-1}=W^{N+j}+\lambda_j\left(W^{N+j}-W^{N+j-1}\right)$. Repeating the proof of \citet[Theorem 7.4]{garrigos2023handbook}, we know that
    \begin{equation*}
        \E[\|Z^{T} - W^*\|^2]\leq \|Z^{-1} - W^*\|^2 - 2\eta\lambda_{T}\E[f(W^{(k+1)(T+1)})-f^*]+ 2\eta\lambda_{0}\E[f(W^{k(T+1)})-f^*] + 2\eta^2 G^2 (T+1)
    \end{equation*}
    i.e
    \begin{equation*}
        2\eta\lambda_{T+1}\E[f(W^{(k+1)(T+1)})-f^*]\leq \|Z^{-1} - W^*\|^2 - \E[\|Z^{T} - W^*\|^2] + 2\eta\lambda_{0}\E[f(W^{k(T+1)})-f^*] + 2\eta^2 G^2 (T+1)
    \end{equation*}
    thus
    \begin{equation*}
        \E[f(W^{(k+1)(T+1)})-f^*]\leq \frac{\|Z^{-1} - W^*\|^2}{2\eta\lambda_{T+1}} - \frac{\E[\|Z^{T} - W^*\|^2]}{2\eta\lambda_{T+1}} + \frac{\lambda_{0}}{\lambda_{T+1}}\E[f(W^{k(T+1)})-f^*] + \frac{\eta G^2 T}{\lambda_{T+1}}
    \end{equation*}
    i.e.,
    \begin{equation*}
        \E[f(W^{(k+1)(T+1)})-f^*]\leq \frac{\|Z^{-1} - W^*\|^2}{\eta(k+1)(T+1)} - \frac{\E[\|Z^{T} - W^*\|^2]}{\eta(k+1)(T+1)} + \frac{k}{k+1}\E[f(W^{k(T+1)})-f^*] + \frac{2\eta G^2}{(k+1)}
    \end{equation*}

    Now sum above equations up for $k=0,1,2,...,K-1$ (since $Z^{-1}=W^{N}=W^{kT}$), 
    \begin{equation*}
        \E[f(W^{K(T+1)}) - f^*]\leq \sum_{k=0}^{K-1}\frac{\|W^{kT} - W^*\|^2}{\eta(k+1) (T+1)} + \sum_{k=0}^{K-1}\frac{2\eta G^2}{k+1}.
    \end{equation*}
    
\end{proof}

\end{document}